\documentclass[journal]{IEEEtran}
\usepackage{url}
\usepackage{cite}
\usepackage{graphicx}
\usepackage{subfig}
\usepackage{color}
\usepackage[pdftex]{hyperref}
\usepackage{amsmath}
\usepackage{algorithm,algorithmic,amssymb}
\usepackage{amsthm}
\usepackage{tabularx}
\usepackage[normalem]{ulem}
\newtheorem{theorem}{Theorem}[section]

\renewcommand{\vec}[1]{\boldsymbol{\mathbf{#1}}}
\newcommand{\R}{\mathbb{R}}

\begin{document}
%
\title{A Compressed Sensing Based Decomposition of Electrodermal Activity Signals}

\author{Swayambhoo~Jain,
		Urvashi~Oswal,
        Kevin~S.~Xu,
        Brian~Eriksson,
        and~Jarvis~Haupt
        \thanks{SJ, JH are with the Department of Electrical and Computer Engineering, University of Minnesota -- Twin Cities, UO is with the Department of Electrical and Computer Engineering, University of Wisconsin -- Madison, KSX is with the Electrical Engineering and Computer Science Department, University of Toledo, BE is with Technicolor Research -- Los Altos. Author emails: {\tt \{jainx174, jdhaupt\}@umn.edu, uoswal@wisc.edu, kevin.xu@utoledo.edu, brian.eriksson@technicolor.com }}}

\maketitle

\begin{abstract}
The measurement and analysis of Electrodermal Activity (EDA) offers applications in diverse areas ranging from market research, to seizure detection, to human stress analysis. Unfortunately, the analysis of EDA signals is made difficult by the superposition of numerous components which can obscure the signal information related to a user's response to a stimulus.  We show how simple pre-processing followed by a novel compressed sensing based decomposition can mitigate the effects of the undesired noise components and help reveal the underlying physiological signal.  The proposed framework allows for decomposition of EDA signals with provable bounds on the recovery of user responses.  We test our procedure on both synthetic and real-world EDA signals from wearable sensors and demonstrate that our approach allows for more accurate recovery of user responses as compared to the existing techniques.
\end{abstract}


\IEEEpeerreviewmaketitle

\section{Introduction}
Electrodermal Activity, or EDA, is typically recorded as the conductance over a person's skin, near concentrations of sweat glands ({\em e.g.,} palm of the hand or finger tips~\cite{Taylor13}). EDA signals have been shown to include significant information pertaining to human neuron firing~\cite{nishiyama01} and psychological arousal~\cite{sidis10}.  While previously a signal that was only practically measured in a controlled laboratory setting, recent wearable devices, such as the Affectiva Q sensor~\cite{affectivaWhite} and the Empatica E4 sensor~\cite{empatica}, offer the ability to non-invasively measure EDA signals in real-world environments. 

An EDA signal is generally characterized by a slowly changing Skin Conductance Level (SCL) combined with several short-lived Skin Conductance Responses (SCRs).  The physiological explanation can be summarized as follows: the SCL is measuring the overall absorption of sweat in the user's skin, while each SCR is measuring a discrete event of sweat expulsion triggered by user excitement or psychological arousal in response to stimuli~\cite{benedek10}. We refer to these discrete events as \emph{SCR events}. The primary focus of prior EDA signal analysis has been to extract the informative SCR events from the observed signals, due to applications ranging from content valence classification~\cite{silveira13}, to audience cohort analysis~\cite{lian14}, to stress detection~\cite{lu12}.  This can prove to be quite challenging due to the overlap of SCR signal components, a dominant SCL  signal, signal artifacts due to motion, and the inclusion of measurement noise.  As a result, there are a large number of proposed techniques to extract SCR events from observed EDA signals~\cite{lim97, alexander05, benedek10, bach10, silveira13, cvxEDA, chaspari2015sparse}, which are discussed in detail later in the paper.

Unfortunately, these prior techniques have a series of drawbacks.  First, many of these techniques perform only simple heuristic-based approaches to extract the SCR events, which causes the techniques to be sensitive to noise and motion artifacts, \textit{i.e.}~sudden shifts in skin conductance due to changes in the position of the sensor.  Second, these techniques lack error bounds on the recovered SCR events, so there is no guarantee for accuracy.  Finally, most of the prior methods have ignored the contribution of motion artifacts.  As EDA becomes more commonly observed via wearable devices, it is more important to mitigate such motion artifacts.

In this paper, we offer a new, more realistic EDA signal model that considers the observed EDA signal as the superposition of a  { \emph{baseline} signal (signal component due to SCL changes and motion artifacts)}, informative SCR components, and measurement noise.  Given this cluttered observed signal, we discuss how existing signal de-mixing work ({\em e.g.,}~\cite{mccoy2013achievable, mccoy2014sharp}) indicates significant challenges in reliably extracting our desired sparse SCR event signal.  We overcome these challenges by providing a new signal model for the baseline signal component which captures changes in measured skin conductance due to motion as well as changes in SCL.  Further, we exploit this signal structure by a simple pre-processing step, which transforms this recovery problem into the more tractable problem of sparse deconvolution in the presence of bounded noise.

The problem of sparse deconvolution has been examined extensively in the compressed sensing literature ({\em e.g.,}~\cite{donoho2006compressed, haupt2010toeplitz,romberg2009compressive,rauhut2012restricted, yin2010practical,berger2010sparse}).  We show how our EDA problem setup requires additional changes to the standard compressed sensing problem. We use modified compressed sensing tools to estimate the SCR events using a concise optimization program and corresponding recovery error bounds. This results in ``first-of-its-kind'' EDA signal decomposition with known error rates.

We test this methodology on a series of both synthetic and real-world EDA signals. Using synthesized data we are able to sweep varying noise and sparsity levels to reveal regimes where our technique accurately recovers the sparse responses.  We then show on real-world EDA signals that user reactions to simple stimuli can be extracted with high accuracy compared with existing EDA decomposition algorithms.

The rest of the paper is organized as follows.  We review prior work on EDA signal analysis and compressed sensing in Section~\ref{sec:rel}.  Our refined model for observed EDA signals is detailed in Section~\ref{sec:model}.  Error bounds for our compressed sensing approach on EDA signals are shown in Section~\ref{sec:edaSingle}.  Experiments on both synthetic and real-world EDA signals are shown in Section~\ref{sec:exp}.  Finally, we conclude and discuss future work in Section~\ref{sec:conclusions}.

\section{Related Work}
\label{sec:rel}

The study of Electrodermal Activity signals, or EDA signals, dates back to the early $20^{\text{th}}$ century ({\em e.g.,} \cite{sidis10}) with the observation of a connection between changes in user skin conductance and psychological state.  In recent years, this connection has been validated by examining brain function via fMRI and skin conduction via EDA concurrently in~\cite{critchley00}, and by showing the specific regions of the brain that correspond with EDA changes and video recordings of sweat glands in~\cite{nishiyama01}.  The promise of EDA as a window into user psychology resulted in extensive work on evaluating the connection between EDA and user interactions~\cite{healey10}, stress detection~\cite{lu12}, content and audience segmentation~\cite{lian14}, and reaction to video content~\cite{silveira13}---to name only a few.

Applications using EDA signal analysis rely on the extraction of a user's fine-grained  responses embedded in the EDA signal called Skin Conductance Responses, or SCRs. These SCRs measure the expulsion of sweat triggered by a user's spike-like stimulus responses, which we call SCR events. SCR events are not explicitly observed in the EDA signal; we observe only the SCRs, which can be modeled as the convolution of the SCR events with a distinguishing impulse response. Significant prior literature has focused both on how to model the SCR impulse response and extract the SCR events from the observed EDA signal.  Examples include a parametric sigmoid-exponential model~\cite{lim97}, a bi-exponential impulse response~\cite{alexander05}, nonnegative deconvolution~\cite{benedek10}, and a variational Bayesian decomposition methodology~\cite{bach10}. These prior techniques are limited by either computational complexity~\cite{bach10} or overly simple models that ignore or heuristically remove additional EDA signal components, such as the SCL, that disguise the SCR events \cite{benedek10,alexander05}.  

The authors of \cite{benedek10} treat the SCL as a constant estimated by averaging the skin conductance signal over the time windows when the estimated SCR (by deconvolution) is below a certain amplitude.
The work of \cite{silveira13} presented a methodology to extract relevant SCR events while considering the SCL signal, but their matching pursuit-based technique used only a rough heuristic to remove this additional signal by deleting the two coarsest-scale components of a discrete-cosine transform applied to the skin conductance. 

More recent work has incorporated SCL in a more principled manner into the EDA signal model. The sparse representation of SCR signal was exploited in \cite{chaspari2015sparse}. In this work, the SCL signal was modeled as a slowly varying linear signal, and the SCR signal was modeled as a sparse linear combination of atoms of a dictionary containing time shifts of variety of function shapes. A greedy method exploiting the sparsity was also proposed for extracting the SCR events signal. Recently, the authors of \cite{cvxEDA} proposed an approach which exploited sparsity from a Bayesian perspective in which the SCL signal was modeled as a sum of cubic B-spline functions, an offset and a linear trend, whereas the SCR signal was modeled by a sparse signal in the dictionary obtained by shifts of bilinear transformations of a Bateman function. Following the maximum \textit{a posteriori} (MAP) estimation principle, a convex formulation was obtained which can be solved efficiently. In contrast to these works \cite{chaspari2015sparse,cvxEDA} we propose a model for the baseline signal that  incorporates shifts in skin conductance due to changes in the positioning of the sensors due to motion, which is crucial when data is collected using wearables. 

Given the sparse nature of the SCR events signal, in order to obtain bounds on our recovery, we leverage literature on compressed sensing \cite{donoho2006compressed}.  Usually focused on sparse signal inference after transformation by random sensing matrices, here we are informed by recent work on sparse deconvolution in a compressed sensing regime \cite{haupt2010toeplitz}, de-mixing of structured signals~\cite{mccoy2013achievable, mccoy2014sharp}, and corrupted sensing for signals with known structure~\cite{foygel2014corrupted}.  Our analysis differs from this prior work via the inclusion of a baseline signal model.  
This requires significant reformulation of the problem to develop new theory and recovery methodologies. 

\section{Model}
\label{sec:model}

The observed EDA skin conductance signal is typically characterized by two dominant components.  The first is a slowly varying Skin Conductance Level (SCL), also referred to as the ``tonic'' component.  The second component is the observation of multiple Skin Conductance Responses (SCRs) arising each from a corresponding SCR event. This component is sometimes referred to as the ``phasic'' component.  These two signal types are detailed in Figure~\ref{Fig:model_fig}. 

\begin{figure}[tb]
\begin{center}
\includegraphics[scale=0.70]{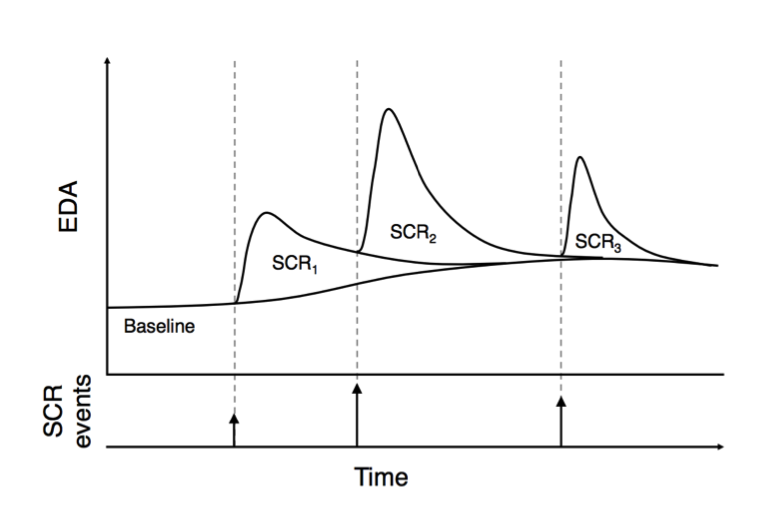}
\end{center}
\caption{ An example of EDA signal where the Skin Conductance Responses
(SCRs) resulting from SCR events signal are shown \cite{silveira13}.} \label{Fig:model_fig}
\end{figure}

\begin{table*}[tb]
\caption{EDA Signal Notation Summary}
\label{table:notation}
\begin{center}
\renewcommand{\arraystretch}{1.2}
\begin{tabular}{|c|c|l|}
\hline
Component & Model Notation & Description \\ \hline
\textbf{Baseline} & ${\vec{b}}$ & {\em Baseline Signal} - Slowly varying skin conductance level with jump discontinuities due to motion\\
\textbf{SCR Events} & ${\vec{x}}$ &  {\em Skin Conductance Response Events} - Signal of sparse stimulus response events from the user \\
\textbf{SCR} & ${\vec{h}*\vec{x}}$ & {\em Skin Conductance Response} - Measured sweat expulsion resulting from the SCR events \\
\textbf{Noise} & ${\vec{n}}$ & Additive noise observed from measurement process and model mismatch\\
\hline
\end{tabular}
\end{center}
\end{table*}

\begin{figure*}
	\begin{center}
		\includegraphics[scale=0.55]{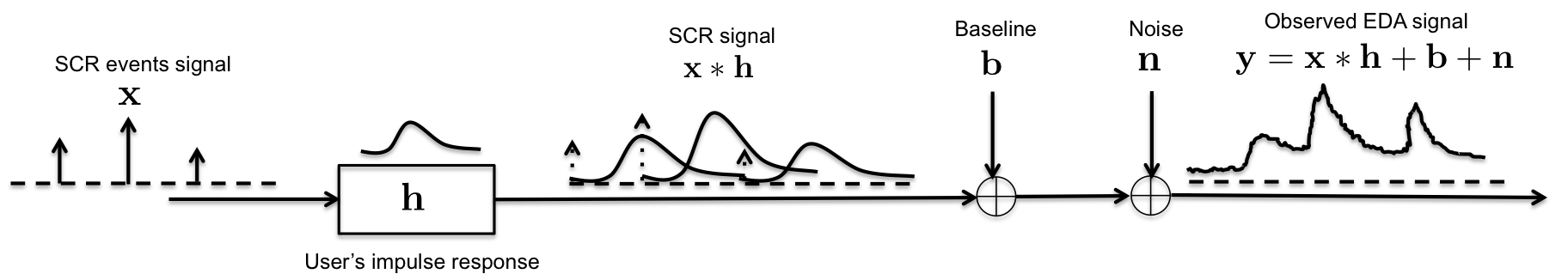}	
	\end{center}
	\caption{Observation model showing the various components in the observed EDA signal.}\label{Fig:Obs_model}
\end{figure*}

The user's physiology explains the existence of these two signal components.
The SCRs are driven by occurrences of SCR events, a sparse selection of events where the user has responded with psychological arousal or excitement to stimulus. The SCR events signal is denoted by the impulse train at the bottom of Figure~\ref{Fig:model_fig}. Prior research in the psychophysiology community ({\em e.g.,}~\cite{benedek10}) has recognized that these SCR events ({\em i.e.,} user excitement events) are correlated with sudomotor neuron bursts, resulting in a user's eccrine glands to expel sweat.  This sweat causes changes in skin conductance in the form of an SCR observation in the shape similar to that shown in Figure~\ref{Fig:model_fig}.  This shape is the result of expelling, pooling, and evaporation of sweat on the surface of the user's skin.  

Additionally, this act results in some sweat being absorbed into the surface of the user's skin, which affects the SCL. We consider the SCL to be a slowly varying signal.  The SCL signal can be changed by temperature, humidity, and other environmental factors along with the physiology of the user ({\em e.g.,} thickness of the user's skin).

In addition to the SCL, there may also be sudden shifts in the skin conductance caused by changes in the positioning of the sensors or the amount of contact of the sensors with the skin, especially in the wearable sensor setting.
Such changes are often reflected by jump discontinuities in the skin conductance. 
We account for such discontinuities, as well as the SCL, in what we call the \emph{baseline} signal component.

\subsection{Model Definition}
Let us consider an observed EDA signal, $\vec{y}$, discretized into $T$ time steps. At each time step there is the possibility of an SCR event. We denote the SCR events signal corresponding to this content by a vector $\vec{x} \in \mathbb{R}^T$, where each component represents the intensity of the user's reaction to the $T$ possible events. Whenever the user has an SCR event, prior research has shown ({\em e.g.,}~\cite{lim97,benedek10,alexander05}) that there are typical ways in which the EDA measurements record conductance changes. We denote this typical sweat response of an user by a vector $\vec{h} \in \mathbb{R}^t$.   In the past \cite{benedek10,alexander05}, the resulting SCR signal has been modeled as a linear time-invariant (LTI) system where the SCR events signal $\vec{x}$ is convolved with the sweat response signal $\vec{h}$ which we denote as $\vec{h} \ast \vec{x} \in \mathbb{R}^{t+T-1}$. 

As mentioned earlier, the SCR signal $\vec{h} \ast \vec{x}$ is superimposed with a baseline signal consisting of SCL and motion artifacts. Denote the baseline signal as $\vec{b} \in \mathbb{R}^{t+T-1}$ and the errors arising due  to observation noise and model mismatch as $\vec{n} \in \mathbb{R}^{t+T-1}$. These notations are summarized in Table~\ref{table:notation}. The observed EDA signal can now be represented as 
\begin{equation}\label{eqn:obs_model}
\vec{y} = \vec{h} \ast \vec{x} + \vec{b} + \vec{n}.
\end{equation}
The final observation model is shown in Figure~\ref{Fig:Obs_model}. Given prior work on the shape of the SCR impulse response $\vec{h}$,
we assume that the impulse response is known \textit{a priori} (we discuss the specific choice of $\vec{h}$ in Section~\ref{sec:exp}).  We consider the SCR events signal $\vec{x}$, the baseline $\vec{b}$, and noise $\vec{n}$ to all be unknown.  

In this paper we propose a model for the observed EDA signal $\vec{y}$ that accounts for both the baseline $\vec{b}$ and observation noise $\vec{n}$ in a principled manner.  This requires further specifications on the signals $\vec{x}$, $\vec{b}$,  and noise $\vec{n}$ which we detail in the following.
 
\subsection{SCR Events Signal Model}
Due to physiology, there are limitations to how often humans can generate SCR events. Motivated by this, we impose a sparsity assumption on the SCR events signal. Specifically, we assume that there are no more than $s < T$ events to which a user responds significantly. More formally, the SCR events signal is assumed to lie in the set
\begin{equation}
\mathcal{X}_{\delta}^s = \left\{\vec{x} \,\Big|\, \vec{x} \in \R^T, 
 \left\| \vec{x} - \vec{x}_s \right\|_1 \leq 
\delta \right\}, 
\end{equation}
where $\delta$ is a small constant and $ \vec{x}_s \in \R^T$ with exactly $s$ non-zero components obtained by retaining the $s$-largest magnitude components of $\vec{x}$.

The above set is the collection of vectors which can be approximated within some distance (in terms of the $\ell_1$-norm) $\delta$ from an exactly $s$-sparse signal. Notice that when $\delta = 0$, the above set is the set of $s$-sparse vectors in $\R^{t+T-1}$.
We note that in most prior literature, the model for the SCR events signal is strictly positive. Here we drop this constraint for a simpler analysis of recovery guarantees.  Our experimental results in Section \ref{sec:exp} show that even without positivity constraints comparable performance can be achieved.

\subsection{Baseline Model}
\begin{figure}[t]
\begin{center}
\includegraphics[scale=0.70]{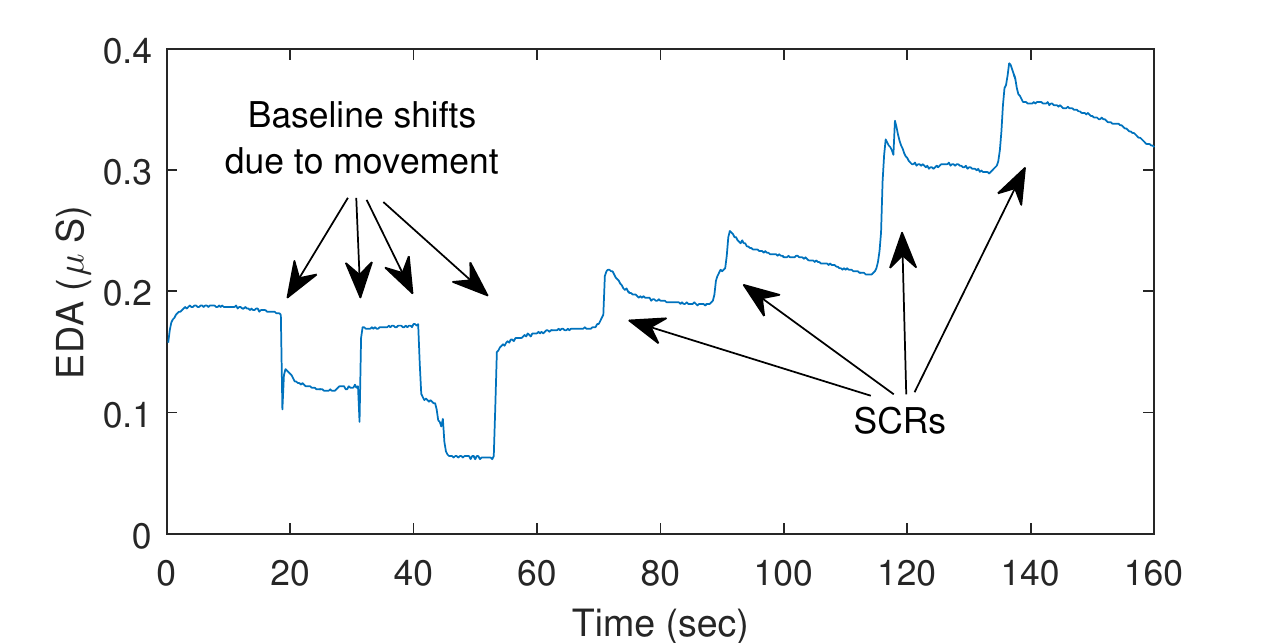}
\end{center}
\caption{An example EDA signal collected using a commercially available wearable EDA sensor showing the impact of baseline shifts due to movement.} \label{Fig:Baseline_Example}
\end{figure}

We propose a novel baseline model, inspired by the wearable setting where changes in the positions of sensors due to movement may lead to rapid changes in the EDA signal. These rapid changes, or baseline shifts, are illustrated in Figure \ref{Fig:Baseline_Example} along with several SCRs. To the best of our knowledge, such baseline shifts have not been examined by previous work on recovering SCR events.
We incorporate these baseline shifts along with the SCL component into a baseline signal $\vec{b}$. We assume $\vec{b}$ changes its magnitude significantly or  has jump discontinuities at no more than $c < t+T-1$ locations. More formally, the baseline signal is assumed to lie in the set
\begin{equation}
\mathcal{B}_{\gamma}^c = \left\{\vec{b} \,\Big|\, \vec{b} \in \R^{t+T-1}, 
 \left\|\vec{Db}-(\vec{Db})_c\right\|_1 \leq
\gamma \right\},
\end{equation}
where    $\vec{D} \in \mathbb{ R}^{(t+T-2) \times (t+T-1) }$ denotes the pairwise difference matrix defined by
\begin{equation}
\label{eqn:pair_diff_mat}
\vec{D} = 
\begin{bmatrix}
1 & -1 & 0 & \cdots & 0 \\
0 & 1 & -1 & \cdots & 0 \\
\vdots & \vdots & \ddots & \ddots & \vdots \\
0 & 0 & \cdots & 1 & -1
\end{bmatrix}
\end{equation}
so that $\mathbf{Db}=[b_1-b_2,b_2-b_3,\ldots,b_{t+T-2}-b_{t+T-1}]$  and   $(\vec{Db})_c \in \R^{t+T-2}$ with exactly $c$ non-zero components obtained by retaining the $c$-largest magnitude components of $\vec{Db}$.
Hence the baseline signal, after pairwise differencing, is assumed to be 
within some distance (in terms of the $\ell_1$-norm) $\gamma$ from 
a $c$-sparse signal.  

\subsection{Bounded Noise Model}
Finally, we consider the additional noise induced by the wearable sensor recording the EDA signals as well as potential model mismatch. Rather than assuming a form for the distribution of this term, we will simply assume that the noise and model inaccuracies are bounded by a fixed value, \emph{i.e.}~$\| \vec{n}  \|_2 \le \epsilon/2$ where $\epsilon > 0$. Here the constant factor $1/2$ is included only to simplify further analysis. 

\subsection{Problem Overview}

The goal of this paper is to obtain the SCR events signal $\vec{x}$ from the EDA observation signal $\vec{y}= \mathbf{h} * \mathbf{x} + \mathbf{b} + \mathbf{n}$ given the prior information that $\vec{x} \in \mathcal{X}^s_\delta $ and $\vec{b} \in  \mathcal{B}^c_\gamma $. 
We assume that the impulse response $\vec{h}$ is known, but the baseline $\vec{b}$, the SCR events signal $\vec{x}$, and the measurement noise $\vec{n}$ are all unknown.

\section{EDA Signal Decomposition}
\label{sec:edaSingle}

The task of recovering the true SCR events $\vec{x}$ from the observed EDA signal $\vec{y}$ is particularly challenging due to the presence of the baseline $\vec{b}$. 
For example, consider the setting when there is an observed signal with no baseline and no noise, {\em i.e.,} $\vec{b}= \vec{0}$, $\vec{n} = \vec{0}$, and $\vec{y} = \vec{h}*\vec{x}$.  The problem of recovering $\vec{x}$ from $\vec{y}$ simply reduces to solving an over-determined linear system of equations given knowledge of $\vec{h}$.  As a result, this problem can be solved with standard deconvolution techniques given very mild assumptions on $\vec{h}$ and without any assumptions needed on true $\vec{x}$.  

In another case, consider there is no baseline but noise is present, {\em i.e.,} $\vec{b} = \vec{0}$, $\vec{n} \neq \vec{0}$, and $\vec{y} = \vec{h}*\vec{x} + \vec{n}$.  This is a standard problem of deconvolution in noise, which in general is a difficult problem to solve.  But, when we consider the added structure of the sparsity of SCR events signal $\vec{x}$, one could exploit this to estimate $\vec{x}$ with provable guarantees.  This setting has been explored in prior work in the field of compressed sensing, {\em e.g.,} \cite{haupt2010toeplitz}.  

\subsection{Dealing with the Baseline Signal}

The main challenge here is the case where the baseline signal is present and non-zero.  One obvious approach could be to consider the baseline as noise and follow previously proposed deconvolution for noisy settings {\em e.g.,} \cite{haupt2010toeplitz}.  However, this would likely fail because the baseline $\vec{b}$ could have very large magnitude.  Our proposed alternative is to exploit the structure of the baseline signal to facilitate the recovery of $\vec{x}$. We linearly transform the baseline signal and jointly recover the transformed baseline and $\vec{x}$. This is often known as a \emph{de-mixing} problem, and there has been recent work on using convex techniques for de-mixing structured signals
\cite{mccoy2013achievable, mccoy2014sharp}. These papers have theoretical  guarantees in terms of statistical dimension. Unfortunately, these guarantees assume a specific random signal generation model which does not hold true for our problem setting. 

Recent work has proposed a \emph{corrupted sensing} approach \cite{foygel2014corrupted} which extends compressed sensing to a setting where observations are corrupted with structured signals. Our problem is different from this setup on two counts: (1) Our sparse signal is convolved with a known SCR impulse response and (2) the  baseline signal in our setting has structure that has not yet been considered in the corrupted sensing literature.  Hence, we leave this as an interesting future direction.

\subsection{EDA Signal Preprocessing}

\begin{figure*}
\begin{center}
\includegraphics[scale=0.50]{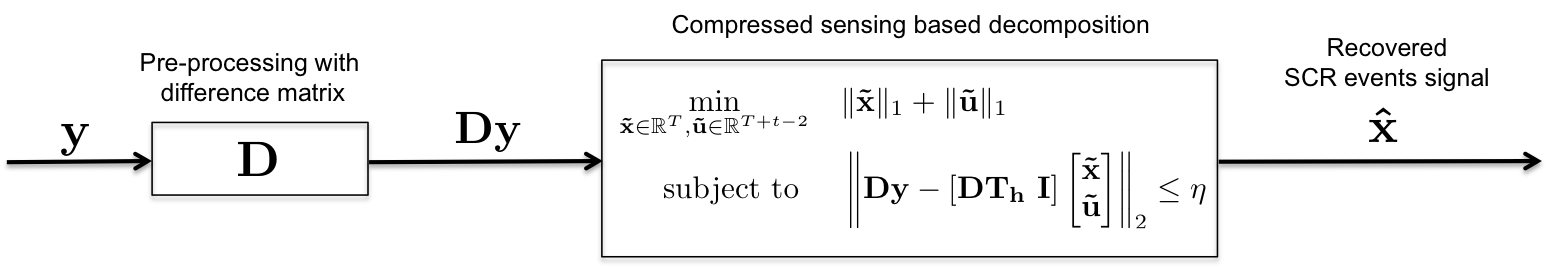}
\end{center}
\caption{Block diagram showing the SCR events signal recovery using compressed sensing based decomposition.} \label{Fig:Proc_tech}
\end{figure*}

We propose an approach that exploits the structure of the EDA signals to mitigate the effects of the baseline signal.  Namely, we can consider that the baseline signals have almost the same consecutive components for most of the signal elements.  As a result, they can be converted to approximately sparse signals by multiplying with the pairwise difference matrix $\vec{D}$ defined in \eqref{eqn:pair_diff_mat}.

Of course, we only have access to the observed signal, $\vec{y}$.  Therefore,
we follow a very simple approach in which we linearly transform the observation $\vec{y}$ using the difference matrix $\vec{D}$ as follows:
\begin{align}
\vec{Dy} = \vec{D}\vec{T_h} \vec{x} + \vec{Db} + \vec{Dn},
\end{align}
where $\vec{T_h}$ denotes a $(t +T -1)\times T$ Toeplitz matrix constructed from a vector $\vec{h} \in \mathbb{R}^t$ and is defined as follows: 
\begin{align}
\mathbf{T_h} = \underbrace{\begin{bmatrix}
	h_1 & 0 & \cdots & 0 \\ 
	h_2 & h_1 & \vdots & \vdots \\
	\vdots & \vdots & \ddots & 0\\
	h_t  & h_{t-1} & \vdots & h_1\\
	0 & h_{t} & \\ 
	\vdots & \vdots & \ddots \\
	0 & \cdots & \cdots & h_t
	\end{bmatrix} }_{ T \ \textrm{columns} }  
\end{align}
such that the convolution between vectors $\vec{h} \in \mathbb{R}^t$ and $\vec{x} \in \mathbb{R}^{T} $, denoted by $\vec{h} \ast \vec{x}$, is a vector in  $\mathbb{R}^{t+T-1}$ and can be written in terms of matrix-vector multiplications as $ \vec{h} \ast \vec{x} = \mathbf{T_h} \vec{x}$. 

With this transformation, the modified baseline signal $\vec{Db}$ is approximately sparse because of the structure of $\vec{b} \in \mathcal{B}_\gamma^c$.  Due to this sparsity, the transformed baseline signal has similar structure to the true SCR events signal $\vec{x}$.  We leverage this fact to jointly estimate $\vec{x} $ and $\vec{Db}$.  Rearranging this term, the observation model becomes 
\begin{align*}
\vec{Dy} = \begin{bmatrix}
\vec{DT_h} & \vec{I}
\end{bmatrix} \begin{bmatrix}
\vec{x} \\ \vec{Db}
\end{bmatrix}+ \vec{Dn},
\end{align*}
where $\vec{I}$ denotes the identity matrix. We have transformed this problem into estimating a vector that is approximately sparse with $s+c$ significant components in $\mathbb{R}^{t+T -2}$, where $s$ is the number of significant non-zero elements in $\vec{x} $, and $c$ is the number of significant non-zeros in $\vec{Db}$. 

Using recent advances in compressed sensing \cite{studer2014stable}, we propose to solve the following problem to estimate $\vec{x}$ and $\vec{Db}$: 
\begin{equation}\label{prob:relaxed}
\begin{split}
\min_{\mathbf{\tilde x} \in \mathbb{R}^T, \mathbf{\tilde u}} \in \mathbb{R}^{T+t-2}   \quad  &  \| \mathbf{\tilde x} \|_1 + \| \mathbf{\tilde u} \|_1  \\
\textrm{ subject to } \quad &  \left\| \mathbf{Dy} - [ \mathbf{D T_{h}} \ \mathbf{I} ]\begin{bmatrix} \mathbf{\tilde x} \\ \mathbf{\tilde u}  \end{bmatrix}  \right\|_2 \le \eta, 
\end{split}
\end{equation}
where $\eta>0$ is a parameter that can be chosen based on the energy of noise $\vec{n}$ as detailed in the next subsection. The above problem is known to be a convex problem which can be solved by using well-known convex optimization software ({\em{e.g.,}} CVX~\cite{cvx}). The final recovery procedure based on above discussion is summarized in Figure~\ref{Fig:Proc_tech}. We note that our problem has Toeplitz structure which can be exploited for developing computationally efficient algorithm using the ideas from matrix-free convex optimization modeling \cite{diamond2015matrix,becker2012tfocs}. We leave this as an interesting future direction of work.

\subsection{Error Guarantees}
The fundamental question that arises here is how well the estimates obtained by solving above problem work. Specifically, how close is the optimal solution $\mathbf{\hat{x}}$ of \eqref{prob:relaxed} to the true SCR events signal $\mathbf{x}$? We have the following theorem to specifically detail the error in our recovered SCR events signal.
\begin{theorem} \label{thm:error_bound}
	Let $ \mathbf{y} = \mathbf{h} \ast \mathbf{x}  + \mathbf{b}  + \mathbf{n} $, where $ \mathbf{x}  \in \mathcal{X}_\delta^s, \mathbf{b}  \in \mathcal{B}_\gamma^c  $.  Denote $\vec{C} = [ \vec{DT_h} \ \vec{I}  ]$ and define the coherence parameters $\mu_h, \mu_m, \mu_c$ as 
	\begin{align*}
	&\mu_h = \max_{i \ne j} \ \frac{ | \vec{t}_i^T \vec{t}_j   |}{ \|\vec{t}_i\|_2 \|\vec{t}_j\|_2 } , \ \mu_c = \max_{i \ne j} \ \frac{ | \vec{c}_i^T \vec{c}_j   |}{ \|\vec{c}_i\|_2 \|\vec{c}_j\|_2 }\\
	&\mu_m = \max_{i , j} \ \frac{ | \vec{t}_i^T \vec{e}_j   |}{ \|\vec{t}_i\|_2 \|\vec{e}_j\|_2 }
	\end{align*}
	where $\vec{t}_i, \vec{e}_i, \text{ and } \vec{c}_i$ are the $i^{th}$ columns of matrices $\vec{DT_h}, \vec{I}, \text{ and } \vec{C} $, respectively. If  $\|\mathbf{n}\| \le \epsilon / 2$ and 
	\begin{align*}
	 s + c < \textrm{max}\left\{ \frac{2(1+\mu_h)}{\mu_h + 2 \mu_c + \sqrt{\mu_h^2 + \mu_m^2}}    , \frac{1+\mu_c}{2 \mu_c}\right\} , 
	\end{align*}
then the solution $\mathbf{\hat x, \hat u }$ of \eqref{prob:relaxed}
	using $\epsilon \le \eta$ satisfies 
	\begin{align*}
	\left\| \mathbf{x}  - \mathbf{ \hat x} \right\|_2  \le C_1(\epsilon + \eta) + 
	C_2( \delta +  \gamma ) 
	\end{align*}
	where $C_1, C_2>0$ depend on $\mu_c$, $\mu_h$, $\mu_m$, $s$, and $c$. 
\end{theorem}
\begin{proof}
See Appendix.
\end{proof}

\begin{figure*}
	\centering
	\subfloat[]{\includegraphics[width=0.31\linewidth]{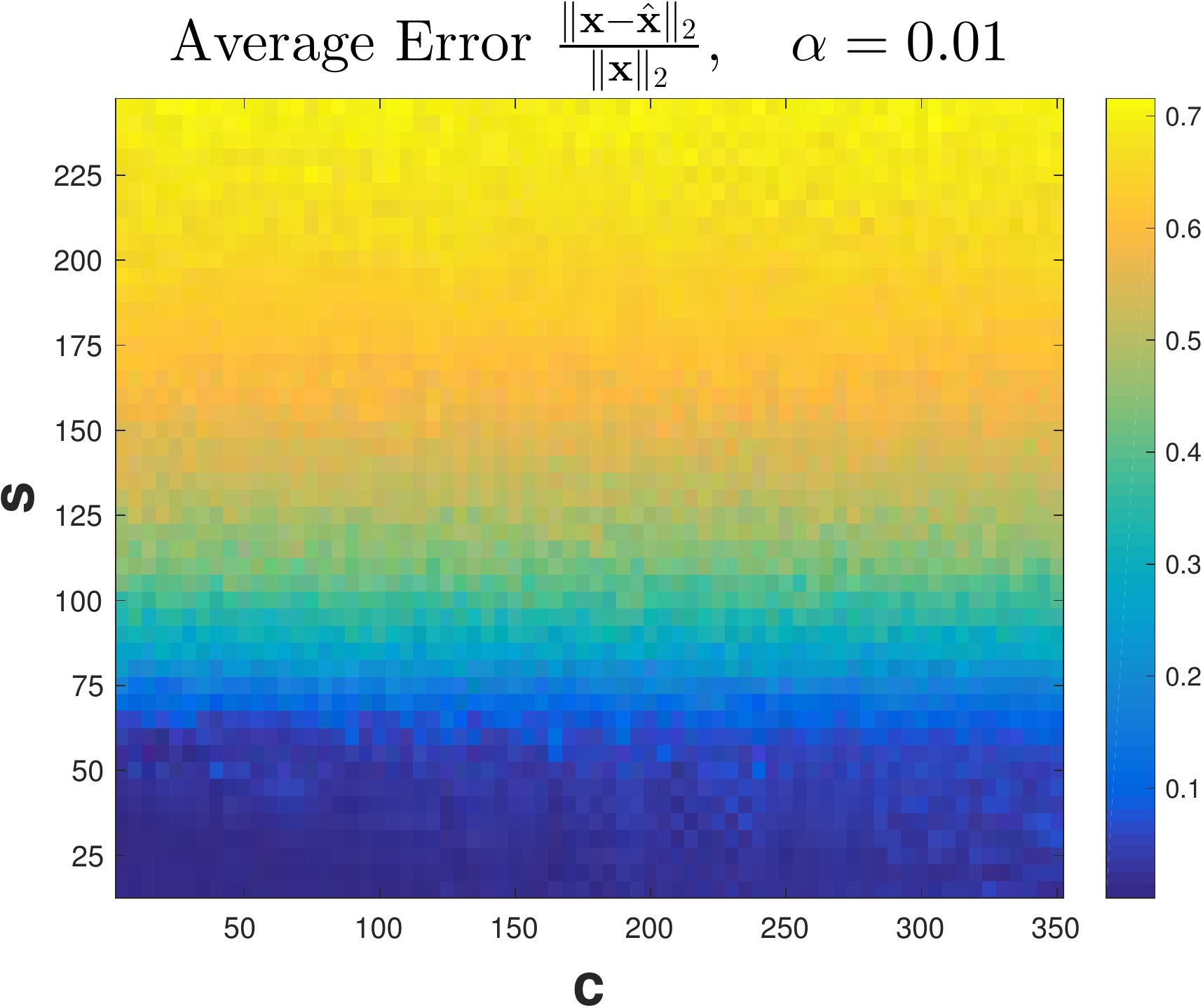} \label{Fig:PhaseDiag001}} \quad
	\subfloat[]{\includegraphics[width=0.31\linewidth]{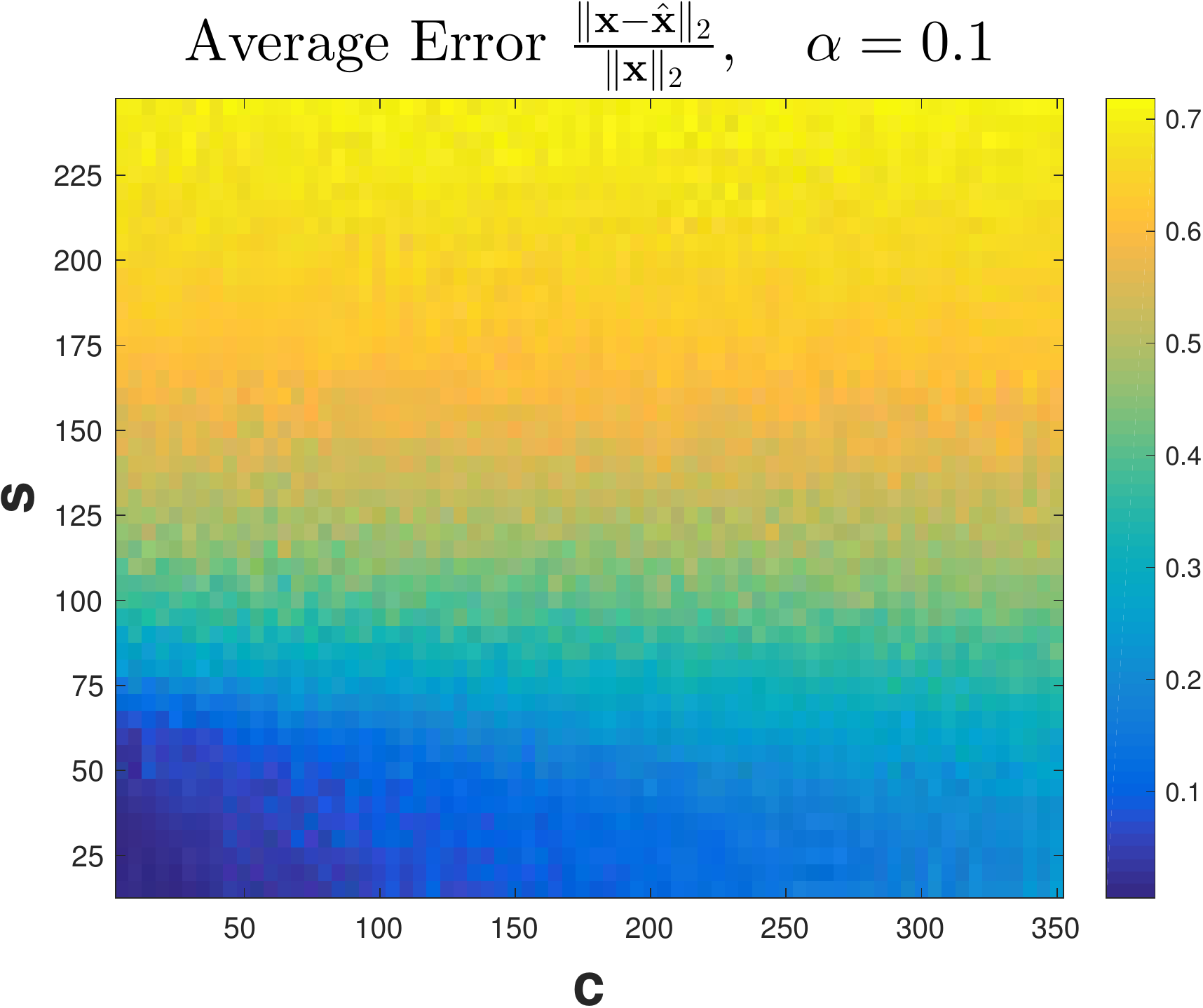} \label{Fig:PhaseDiag01}} \quad
	\subfloat[]{\includegraphics[width=0.31\linewidth]{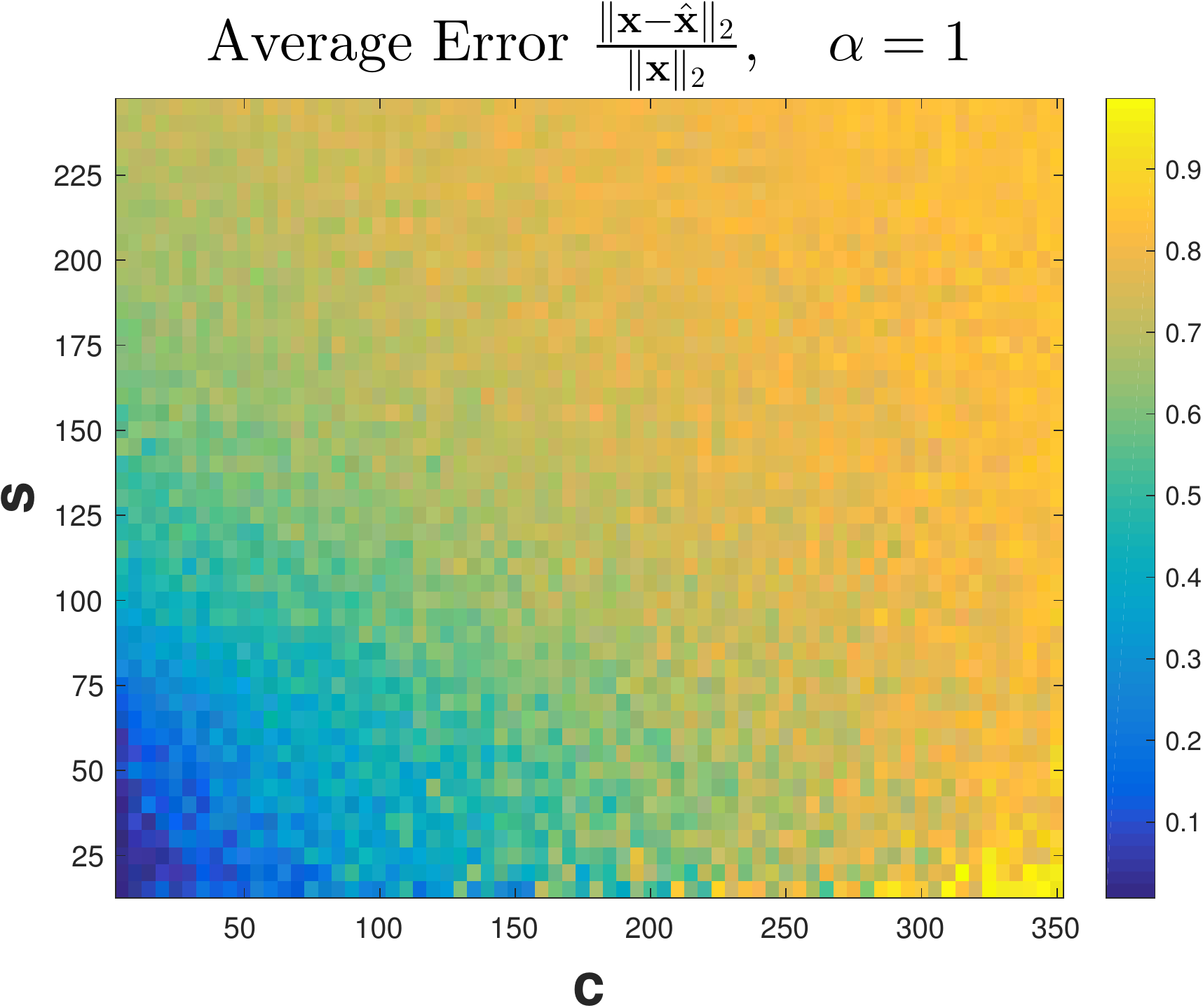} \label{Fig:PhaseDiag1}}
	\caption[]{Estimation error diagrams with synthetic data for various values of number of SCR events $s \in \{5, 10, \dots, 230\}$ and baseline jumps $c \in \{ 5, 10, \dots, 350\}$. Panels \subref{Fig:PhaseDiag001}, \subref{Fig:PhaseDiag01}, and \subref{Fig:PhaseDiag1} correspond to scaling the magnitude of the baseline component using $\alpha = 0.01, 0.1$ and $1$, respectively.} \label{Fig:PhaseDiag}
\end{figure*}

The above theorem states that, when the combined sparsity of the true SCR events signal and the baseline signal after the difference filter is small enough, the estimate of the SCR events signal $\vec{\hat x}$ is accurate. More specifically, the $\ell_2$ norm of the error vector (\textit{i.e.,} the difference between the true and the estimated SCR events signal) is upper bounded by a quantity which is proportional to the constants $\epsilon, \delta $ and $\gamma$, which are part of our signal model, and the optimization parameter $\eta$, provided that it is chosen to be greater than or equal to $\epsilon$. As long as these constants are small, our approach yields an accurate solution. 
In our setting, it is reasonable to assume that these constants are indeed small for the following reasons. The SCR events signal $\vec{x}$ is sparse due to physiological reasons, as previously discussed. The baseline signal should not have too many jump discontinuities provided that the user is not constantly moving the sensor, which causes $\vec{Db}$ to also be sparse.  
Finally, $\epsilon$ depends on the noise power and model mismatch and is small provided that the noise power is much lower than the signal power and that our model assumptions are close to reality.

The terms $C_1$ and $C_2$ are known to decrease with decreasing $s,c$  \cite{studer2014stable}. This implies that the error in the recovery decreases as the signals become more sparse. The range of values of $s+c$ for which the error bounds holds depends on the coherence parameters. These parameters critically depend on the shape and length of $\vec{h}$ which we assume are known.  It is known that with decreasing coherence parameters $\mu_c$, $\mu_h$, and $\mu_m$, the recovery of a sparse signal improves \cite{studer2014stable}. All the coherence parameters can be viewed as the maximum entries of the sub-blocks of the matrix 
\begin{align*}
\vec{G} &= \begin{bmatrix}
(\vec{DT_h \Lambda})^T  \\ 
\vec{I} 
\end{bmatrix}\begin{bmatrix}\vec{DT_h \Lambda} & 
\vec{I} 
\end{bmatrix} - \vec{I} \\
& = \begin{bmatrix} (\vec{DT_h \Lambda})^T  \vec{DT_h \Lambda} - \vec{I} & 
(\vec{DT_h \Lambda})^T \\
\vec{DT_h \Lambda}  & \vec{0}
\end{bmatrix}, 
\end{align*}
where $\vec{\Lambda}$ is a diagonal matrix such that the columns of the matrix $\vec{DT_h \Lambda}$  have unit $\ell_2$ norm. The coherence parameters can be written in terms of sub-blocks of matrix $\vec{G}$ as follows 
\begin{align*}
\mu_h &= \| (\vec{DT_h \Lambda})^T  \vec{DT_h \Lambda} - \vec{I}    \|_{\rm{max}} \\ 
 \mu_m &= \|  \vec{DT_h \Lambda}  \|_{\rm{max}}   \\
 \mu_c &= \max \{ \mu_h, \mu_m \},
\end{align*}
where for a matrix $\vec{X}$, the maximum absolute entry of the matrix is denoted by $\| \vec{X}\|_{\rm{max}}$.

\section{Experiments}
\label{sec:exp}

Using a combination of both synthetic and real-world EDA data, in this section we demonstrate the feasibility and accuracy of our proposed compressed sensing approach to EDA decomposition.  Our synthetic data experiments sweep a wide selection of sparsity values and baseline signal energy levels to demonstrate SCR event recovery accuracy. Using real-world EDA data, we then show how our technique allows for more accurate inference of EDA events signal as compared to prior techniques.

\subsection{Synthetic Data Experiment} 
The first experiment is dedicated to demonstrating the recovery accuracy of our procedure on synthetic data. 
We obtained the impulse response vector $\vec{h}$ by sampling the function $f(u)$ shown in Figure~\ref{Fig:Impulse_Res} at the rate of $4$ samples per second  in the interval $u \in [0,40]$.
This choice of impulse response was informed by prior psychophysiology literature~\cite{alexander05}.
The $\vec{h}$ obtained in such manner lies in $\mathbb{R}^{160}$. We fixed $T=240,\delta = 0.01$ and $\gamma = 0.01$. 

\begin{figure}
\begin{center}
\includegraphics[scale=0.3]{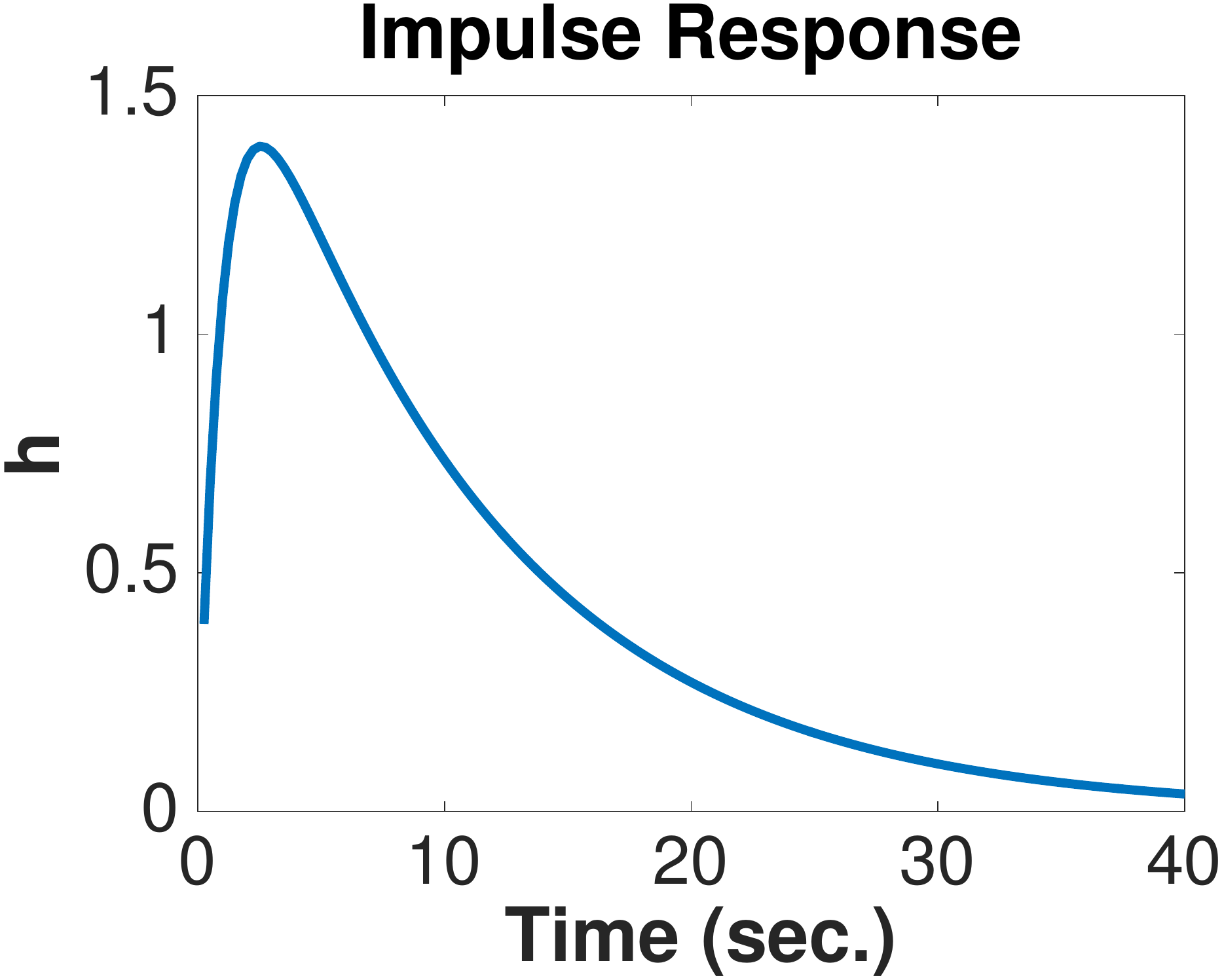}
\end{center}
\caption{ The impulse response $\vec{h}$ was obtained by sampling the function $f(u) =  2\left(e^{-\frac{u}{\tau_1}} -e^{-\frac{u}{\tau_2}} \right) $ for $u\ge 0$ and $f(u) = 0 $  otherwise. Here $ \tau_1 = 10, \tau_2 = 1$ and the is function sampled at the rate of $4$ samples per second in the interval $u \in [0,40]$.} \label{Fig:Impulse_Res}
\end{figure}

\begin{figure*}
	\centering
	\subfloat[]{\includegraphics[width=0.48\linewidth]{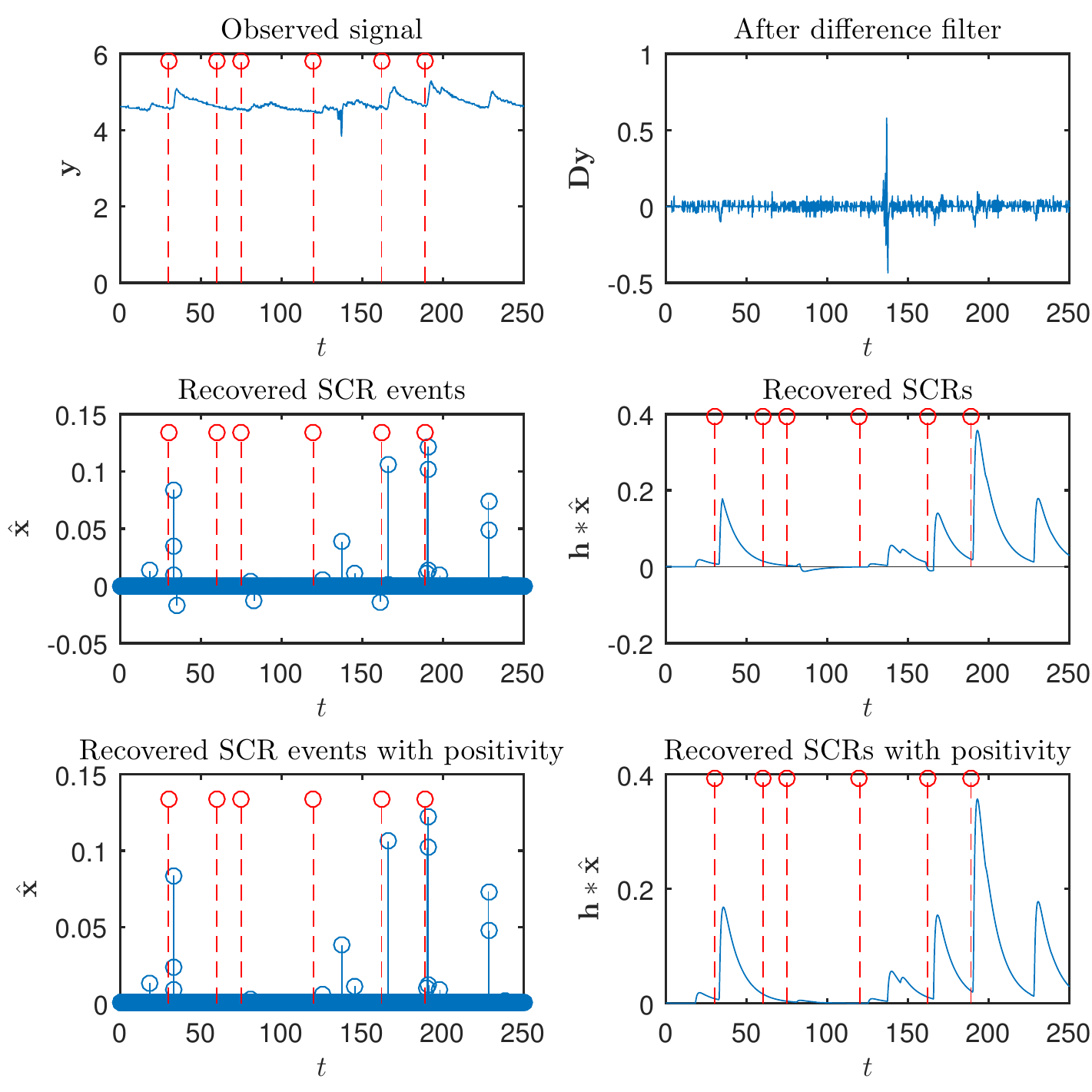}
	\label{Fig:decompose_4}}
	\quad
	\subfloat[]{\includegraphics[width=0.48\linewidth]{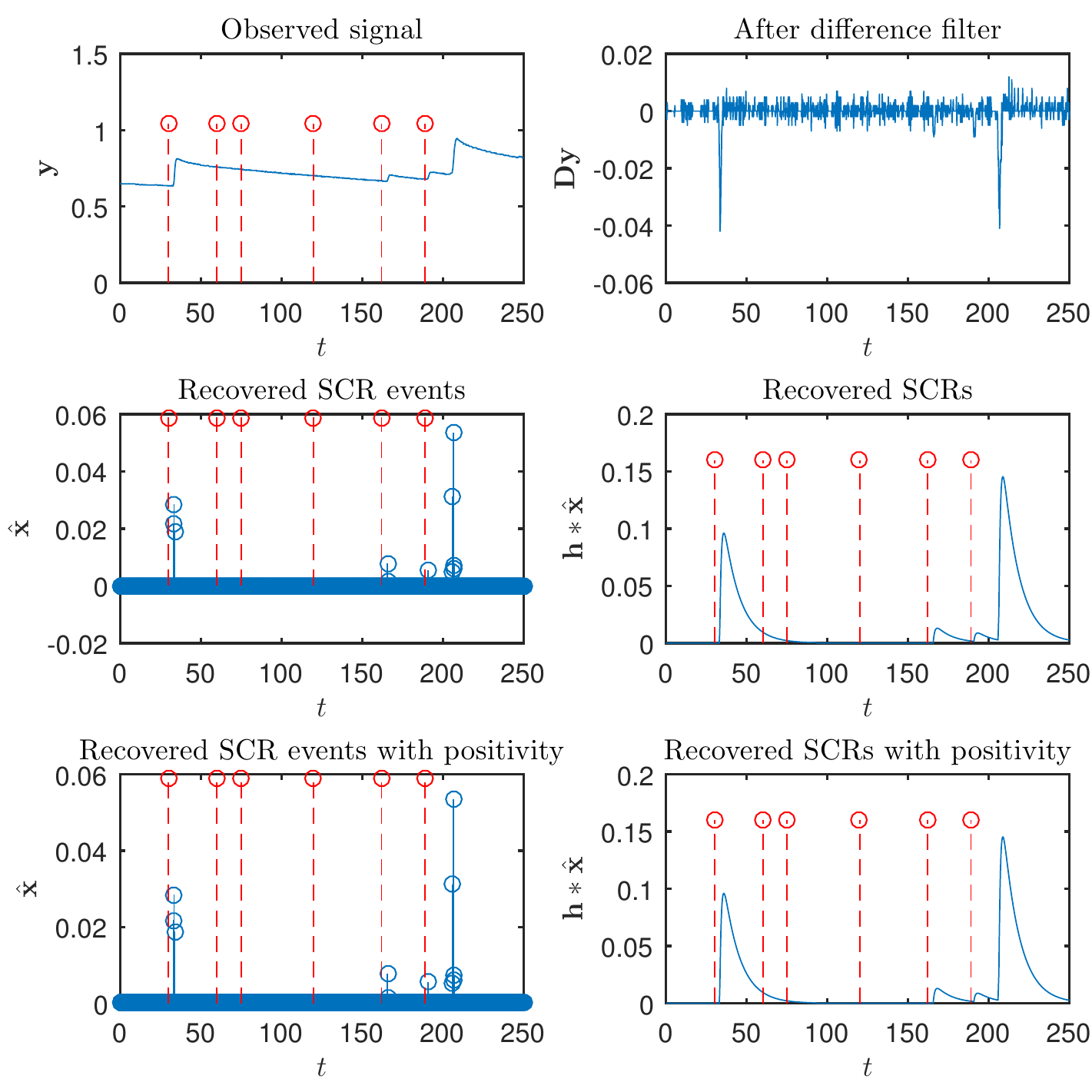}
	\label{Fig:decompose_6}}
	\caption[]{Decomposition of real-world EDA data for two users in \subref{Fig:decompose_4} and \subref{Fig:decompose_6} respectively. Stimuli are presented to the users at moments denoted by red dotted vertical lines. We show results for our compressed sensing approach with and without positivity constraints for data downsampled to 4 Hz.} \label{Fig:decompose}
\end{figure*}

For a given number of SCR events $s$ and number of baseline jumps $c$, we randomly generate $\vec{x} \in \mathcal{X}_{\delta}^s$ and $\vec{b} \in \mathcal{B}_{\gamma}^c$. A random $\vec{x} \in \mathcal{X}_{\delta}^s$  is generated  by first choosing the $s$ significant components uniformly at random and filling these components with a random vector  in $\mathbb{R}^s$ with i.i.d. exponentially distributed entries with mean $2$. This is followed by adding to it a rescaled standard Gaussian random vector in $\mathbb{R}^T$ with $\ell_1$ norm $\delta$. Similarly, a random $\vec{Db}$ was generated by first choosing the $c$ significant components uniformly at random and filling each of these components with a standard Gaussian variable followed by adding a rescaled standard Gaussian random vector in $\mathbb{R}^{t+T-2}$ with $\ell_1$ norm $\gamma$. Using these steps we generate the observations as follows: 
\begin{align}
 \vec{Dy} = \vec{D T_h x} + \alpha \vec{Db} + \vec{n},
\end{align}
 where $\vec{n}$ is also a rescaled Gaussian random vector with $\ell_2$ norm equal to $\epsilon=0.01$.  We generate multiple experiments using different values of $\alpha$, a scaling factor applied to $\vec{Db}$ relative to $\vec{D T_h x}$. These observations are then used to obtain the estimate $\vec{\hat{x}}$ by solving the problem in \eqref{prob:relaxed}  with $\eta = 1.05 \epsilon$.

Figure~\ref{Fig:PhaseDiag} shows the average relative estimation error $\frac{\|\vec{x - \hat{x}}\|_2}{\|\vec{x}\|_2}$, where the average is obtained by $30$ random observations for various values of  $s$ and $c$.   For baseline components with low energy in Figure~\ref{Fig:PhaseDiag001}, we find that the ability to recover is almost entirely dependent on the number of SCR events embedded in the generated EDA signal.  Regardless of the number of baseline jumps, we find that for fewer than $75$ SCR events in an EDA signal, we can accurately recover the SCR signal.  On the other hand, as the energy in the baseline increases, as shown in Figures~\ref{Fig:PhaseDiag01} and \ref{Fig:PhaseDiag1}, we find that a large number of jumps in the baseline signal can degrade our ability to accurately recover the SCR events.

\subsection{Experiments with Real-World EDA Data}

Our second experiment examines the performance of our methodology on real-world EDA signals.  We used EDA signals from a simple video stimulus experiment, originally published in~\cite{silveira13}.  The video consists of six short stimulus clips (each lasting less than 10 seconds) with differing levels of complexity.  Specifically, this video contains a baby crying sound, a gun shot sound, a dog barking sound, the image of a gun, and two short videos of a subject injuring themselves.  This stimulus is interspersed with silence where no audio or video is presented to the user. The EDA data consists of EDA traces from nine subjects (6 male, 3 female, with ages ranging between 20 and 50 years old) who watched the same video content in a darkened environment.  The EDA was recorded using the Affectiva Q Sensor~\cite{affectiva} with sampling at 32 Hz.

Unlike with the synthetic data experiment, we cannot assess relative estimation error $\frac{\|\vec{x - \hat{x}}\|_2}{\|\vec{x}\|_2}$ because we do know the magnitudes of the ground-truth SCR events $\vec{x}$. We do, however, know the times at which the stimulus clips and periods of silence were presented to the users. Very few SCR events should occur during the periods of silence, while many SCR events should occur during the stimulus clips, thus we can use these times to assess how well our EDA decomposition technique is able to detect SCR events. 
Specifically, we used 10 second windows around each stimulus and silence clip, and then aggregated the estimated SCR event coefficients between the start of the clip and the end of the clip.  These aggregated values are then compared to a threshold to produce a binary decision as to whether SCR events are present in the time window. The impulse response vector $\vec{h}$ was obtained by sampling the function $f(u) = 2(e^{-\frac{u}{\tau_1}} -e^{-\frac{u}{\tau_2}})$ for $u\ge 0$ and $f(u) = 0 $ otherwise. We chose $\tau_1 = 10, \tau_2 = 1 $.  For our proposed technique and defined $\vec{h}$, we obtained estimates of SCR events signal for each user by solving \eqref{prob:relaxed} with $\eta = 0.14$.

To evaluate our performance we use four alternative methodologies:  (1) aggregated raw EDA signal for each user in the stimulus and silence time windows, (2) the non-negative deconvolution analysis technique of Benedek and Kaernbach~\cite{benedek10} using the Ledalab software package~\cite{ledalab}, (3) the convex optimization approach cvxEDA proposed in \cite{cvxEDA}, and (4) a modification of our approach with positivity constraint for the SCR events signal\footnote{Specifically, we solve problem \eqref{prob:relaxed} with positivity constraint $\mathbf{x} \ge 0$.}. The raw EDA analysis will communicate if the mean EDA signal is informative with respect to our stimulus, while the deconvolution approach demonstrates EDA decomposition that ignores the prominent baseline signal. The cvxEDA approach will compare our proposed model with a recent EDA decomposition technique using convex optimization. The approach with positivity constraints will test whether including positivity constraints in our problem setup improves recovery accuracy. 

We perform experiments on the original 32 Hz data as well as 4 Hz and 8 Hz downsampled versions, which are more in-line with the sampling rates of commercially available wearable sensors such as the Empatica E4 \cite{empatica} and Microsoft Band 2 \cite{MicrosoftBand2} (4 and 5 Hz, respectively). 
For cvxEDA, the same values $\tau_1=10$ and $\tau_2=1$ as for our approach were used\footnote{cvxEDA also requires specification of the sampling interval $\delta$, which was set to $1/\text{sampling frequency}$, and other parameters $\delta_0$ , $\alpha$, and $\gamma$, which were set to the default values in the software package.}, whereas for Ledalab, $\tau_1$ and $\tau_2$ were automatically optimized by the software package.

\textit{Discussion of results:} The result of signal decomposition on  the 4 Hz downsampled signal is shown in Figure \ref{Fig:decompose}.  In this figure we highlight the recovered signals with our approach and a modified version with positivity constraints on the SCR events signal. Figures \ref{Fig:decompose_4} and \ref{Fig:decompose_6} correspond to two different users that were chosen at random from our data set. Stimuli are presented to the users at moments denoted by red dotted vertical lines. We see that the recovered SCR events signal is similar for both techniques except for the events with small negative amplitudes when no positivity constraints are enforced. The reconstructed SCR signal $ \mathbf{h} * \hat{ \mathbf{x} }$ using both approaches are also shown. Overall, we find that our proposed approach performs similarly to its variation with positivity constraints.

\begin{figure}[t]
\centering
\includegraphics[width=0.99\linewidth]{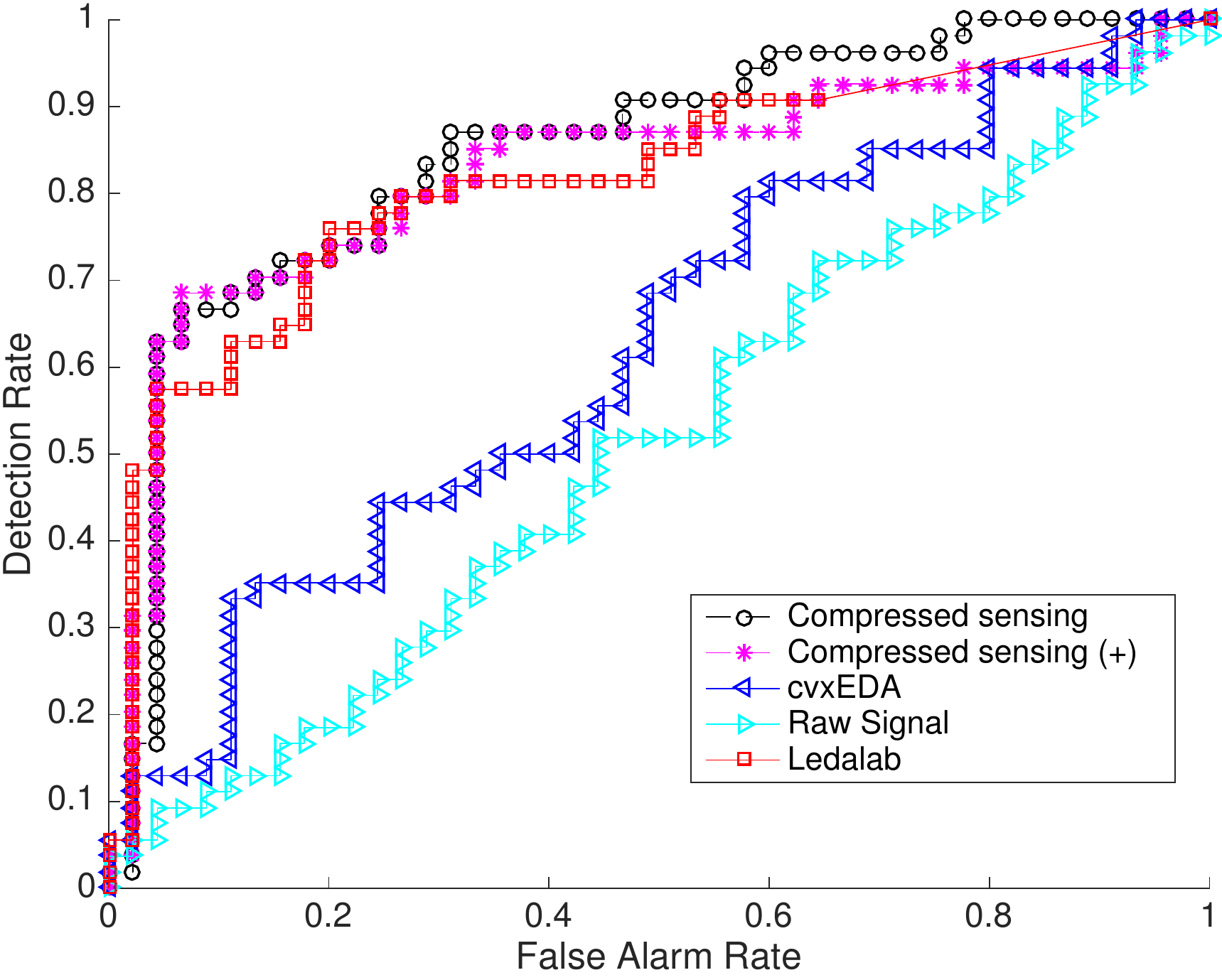} 
\caption{ROC curves for real data SCR event detection experiment at the sampling rate of 4 Hz. Our compressed sensing-based approaches is compared with a variation of our approach with positivity constraints, the non-negative deconvolution approach in Ledalab, the cvxEDA convex optimization based approach, and the raw EDA signal.} \label{Fig:ROC}
\end{figure}

\begin{table}
	\renewcommand{\arraystretch}{1.2}	
	\begin{center}	
	\caption{AUC values for SCR event detection at multiple sampling rates for various approaches on real data experiment.} 
	\label{table:AUC}	
	\tabcolsep=0.18cm
		\begin{tabular}{ | c | c | c | c | c | c |}
			\hline
			  Sampling  & Compressed   & Compressed     &               & Raw     &    \\ 
		     Rate       & Sensing      & Sensing $(+)$  & cvxEDA        & Signal  &  Ledalab       \\ \hline
			4  Hz       & {\bf 0.848}  & 0.825          & 0.622         & 0.539	  &  0.817	        \\
			8  Hz       & {\bf 0.857}  & 0.821          & 0.771         & 0.493	  &  0.824          \\
			32 Hz       & 0.868        & {\bf 0.895}    & 0.819         & 0.514	  &  0.837          \\ \hline
		\end{tabular}	
	\end{center}	
\end{table}

Further, aggregating the accuracy across all nine users, we present the Receiver Operating Characteristic (ROC) curve in Figure \ref{Fig:ROC}, which shows the detection rate for any given false alarm rate at the sampling rate of $4$ Hz. We summarize the ROC curve using the Area Under the Curve (AUC).  We find that our compressed sensing based decomposition (AUC = \textbf{0.848}) and its variation with positivity constraints (AUC = \textbf{0.825}) perform better than both the non-negative deconvolution method in Ledalab (AUC = \textbf{0.817}) and the convex optimization based cvxEDA approach (AUC = \textbf{0.622}). Another insight from these results is that using the raw EDA traces results in accuracy roughly no better than random guessing (i.e., detection rate equal to the false alarm rate), showing the need for processing of the observed EDA signals. 

The results at various sampling rates are shown in Table \ref{table:AUC}. We see that our scheme gives better performance than all other schemes at sampling rates 4 Hz and 8 Hz.  This is an important regime when considering EDA observations from power and storage-constrained wearables.  Our observations also suggest that, at these sampling rates, adding positivity constraints to our approach does not necessarily improve accuracy. In fact, at 4 Hz and 8 Hz, adding positivity constraints actually lowered the AUC.  The only improvements for the positivity constrained techniques was at a sampling rate of 32 Hz.

\section{Conclusions}
\label{sec:conclusions}

In this work we proposed a novel compressed sensing based framework for processing of EDA signals. The proposed framework explicitly models the baseline signal and allows for recovery of the users responses via simple pre-processing followed by compressed sensing based decomposition. We also provided theoretical error bounds on the accuracy of the proposed recovery procedure. Our approach accurately recovers SCR events in experiments on simulated data. Furthermore, our recovery procedure also outperforms existing recovery procedures for an SCR event detection task on real-world EDA data obtained from a video stimulus experiment. 

Future works include considering modified EDA signal models that vary the shape of the impulse response with time and varied noise models, developing computationally efficient algorithms that exploit the Toeplitz structure, exploring the possibility of better recovery guarantees by considering random signal models and with positivity constraints.

\appendix
\begin{proof}[Proof of Theorem \ref{thm:error_bound}]

The proof is a straightforward extension of the following theorem from \cite{studer2014stable}:
\begin{theorem}[~\cite{studer2014stable}, Thm. 4 ] \label{thm:studer}
Let $\vec{t} = \vec{Cw}  + \vec{z}$, with $\vec{C} = [ \vec{A} \ \vec{B}]$, $\vec{w}^{T} = [\vec{x}^{T} \ \vec{u}^{T}]$, and $\| \vec{z} \|_2 \le \epsilon$. 
Define the coherence parameters $\mu_a,\mu_b, \mu_m$, and $\mu_c$ for the dictionary $\vec{C}$ as
\begin{align*}
\mu_a = \max_{i \ne j} \ \frac{ | \vec{a}_i^T \vec{a}_j   |}{ \|\vec{a}_i\|_2 \|\vec{a}_j\|_2 }, \ \mu_b = \max_{i \ne j} \ \frac{ | \vec{b}_i^T \vec{b}_j   |}{ \|\vec{b}_i\|_2 \|\vec{b}_j\|_2 } \\
\mu_m = \max_{i , j} \ \frac{ | \vec{a}_i^T \vec{b}_j   |}{ \|\vec{a}_i\|_2 \|\vec{b}_j\|_2 }, \ \mu_c = \max_{i \ne j} \ \frac{ | \vec{c}_i^T \vec{c}_j   |}{ \|\vec{c}_i\|_2 \|\vec{c}_j\|_2 }
\end{align*}
Assume $\mu_b \le \mu_a$ without loss of generality. If 
\begin{align}\label{eqn:sparsity_bound_thm_studer}
 s + c < \textrm{max}\left\{ \frac{2(1+\mu_a)}{\mu_a + 2 \mu_c + \sqrt{\mu_a^2 + \mu_m^2}}    , \frac{1+\mu_c}{2 \mu_c}\right\} 
\end{align}
then the solution of $\vec{\hat w}$
\begin{align}\label{eqn:opt_prob_thm_studer}
	\min_{\mathbf{\tilde w}}   \quad  &  \| \mathbf{\tilde w} \|_1  \nonumber \\
	\textrm{ subject to } \quad &  \left\| \vec{t} - \vec{C} \vec{\tilde w} \right\|_2 \le \eta, 
\end{align}
using $\epsilon \le \eta$ satisfies 
\begin{align*}
\| \vec{w} - \vec{\hat w}\|_2 \le C_1 (\epsilon + \eta) + C_2 \| \vec{w} - \vec{w}_{n + s} \|_1
\end{align*}
 where  $C_1, C_2 > 0$ depend on  $\mu_a,\mu_b, \mu_m$, $\mu_c$, $s$, and $c$.
\end{theorem} We use the above Theorem \ref{thm:studer} with $\vec{t} = \vec{Dy}, \vec{A} =  \vec{DT_h}, \vec{B} =  \vec{I} , \vec{z} =  \vec{Dn}$, and $\vec{w}^{T} =   [ \vec{x}^{T} \ \vec{u} ]$ with $\vec{u} = \vec{Db}$. First we show that the $\ell_2$ norm of the noise $\vec{z}$ satisfies the assumption in Theorem \ref{thm:studer}. This can be easily seen as follows
\begin{align*}
\| \vec{z} \|_2 &= \| \vec{Dn} \|_2  \\
& \le \| \vec{D} \|_2 \| \vec{n} \|_2 \\
& \le 2 (\epsilon / 2) = \epsilon,
\end{align*}
where the last inequality is due the fact that $\| \vec{D} \|_2 \le 2 $ and $\| \vec{n} \|_2 \le \epsilon / 2 $ by our model assumption. Also, as $\vec{B} = \vec{I}$ is an orthonormal matrix, it is easy to see that $\mu_b = 0$. Since $\mu_a$ is strictly positive under our model assumption, the condition $\mu_b \le \mu_a$ is also satisfied. Further, since we can write $\|\vec{\tilde w} \|_1 = \| \vec{\tilde x} \|_1 + \|\vec{\tilde u}\|_1 $, the optimization problem \eqref{eqn:opt_prob_thm_studer} in Theorem \ref{thm:studer} takes the following form:
	\begin{align*}
	\min_{\mathbf{\tilde x} \in \mathbb{R}^T, \mathbf{\tilde  u} \in \mathbb{R}^{T+t-2}}   \quad  &  \| \mathbf{\tilde x} \|_1 + \| \mathbf{ \tilde u} \|_1  \\
	\textrm{ subject to } \quad &  \left\| \mathbf{Dy} - [ \mathbf{D T_{h}} \ \mathbf{I} ]\begin{bmatrix} \mathbf{ \tilde x} \\ \mathbf{ \tilde  u}  \end{bmatrix}  \right\|_2 \le \eta 
	\end{align*}
The above problem is exactly same as the problem in \eqref{prob:relaxed}, for which error bounds are outlined in Theorem \ref{thm:error_bound}. This essentially establishes that Theorem \ref{thm:studer} can be used to obtain the recovery guarantees of problem \eqref{prob:relaxed}.
Provided that the combined sparsity $s + c$ satisfies condition \eqref{eqn:sparsity_bound_thm_studer} and we choose $\eta$ such that it satisfies $\epsilon \le \eta$, we have the following bound from Theorem \ref{thm:studer}: 
\begin{align*}
\|\vec{w} - \vec{ \hat w} \|_2  &= \sqrt{ \|\vec{x} -  \vec{\hat x}\|_2^2 + \| \vec{Db} - \vec{\hat u}  \|_2^2}  \\
& \le C_1 (\epsilon + \eta) + C_2 \| \vec{w} - \vec{w}_{n+s}\|_1  \\
& \le C_1 (\epsilon + \eta) + C_2 \left\{ \| \vec{x} - \vec{x}_s \|_1 + \| \vec{b}- \vec{b}_c\|_1 \right\}
\end{align*}
Further, combining the above inequality with the fact that
\[
\|\vec{w} -\vec{\hat w} \|_2 =  \sqrt{ \|\vec{x} -  \vec{\hat x}\|_2^2 + \| \vec{Db}- \vec{\hat u}  \|_2^2} \ge \|\vec{x}-  \vec{\hat x}\|_2,
\]
we have arrive at
\begin{align*}
\| \vec{x} - \vec{\hat x}\|_2 \le C_1 (\epsilon + \eta) + C_2 \left\{ \| \vec{x} - \vec{x}_s \|_1 + \| \vec{b}- \vec{b}_c \|_1 \right\},
\end{align*}
which, by our model assumption, can be reduced to
\begin{align*}
\| \vec{x}- \vec{\hat x}\|_2 \le C_1 (\epsilon + \eta) + C_2 (\delta +  \gamma).
\end{align*}
The coherence parameters $\mu_h, \mu_m, \text{ and } \mu_c$ in Theorem \ref{thm:error_bound} are equivalent to coherence parameters $\mu_a, \mu_m $, and $\mu_c$ respectively in Theorem \ref{thm:studer}.
\end{proof}

\bibliographystyle{IEEEtran}  
\bibliography{references}  

\begin{thebibliography}{10}
\providecommand{\url}[1]{#1}
\csname url@samestyle\endcsname
\providecommand{\newblock}{\relax}
\providecommand{\bibinfo}[2]{#2}
\providecommand{\BIBentrySTDinterwordspacing}{\spaceskip=0pt\relax}
\providecommand{\BIBentryALTinterwordstretchfactor}{4}
\providecommand{\BIBentryALTinterwordspacing}{\spaceskip=\fontdimen2\font plus
\BIBentryALTinterwordstretchfactor\fontdimen3\font minus
  \fontdimen4\font\relax}
\providecommand{\BIBforeignlanguage}[2]{{%
\expandafter\ifx\csname l@#1\endcsname\relax
\typeout{** WARNING: IEEEtran.bst: No hyphenation pattern has been}%
\typeout{** loaded for the language `#1'. Using the pattern for}%
\typeout{** the default language instead.}%
\else
\language=\csname l@#1\endcsname
\fi
#2}}
\providecommand{\BIBdecl}{\relax}
\BIBdecl

\bibitem{Taylor13}
{N.~Taylor and C.~Machado-Moreira}, ``{Regional variations in transepidermal
  water loss, eccrine sweat gland density, sweat secretion rates and
  electrolyte composition in resting and exercising humans},'' \emph{Extrem.
  Physiol. Med.}, vol.~2, p.~4, 2013.

\bibitem{nishiyama01}
{T.~Nishiyama} \emph{et~al.}, ``Irregular activation of individual sweat glands
  in human sole observed by a videomicroscopy,'' \emph{Autom. Neurosci.: Basic
  Clin.}, vol.~88, pp. 117--126, 2001.

\bibitem{sidis10}
{B.~Sidis}, ``The nature and cause of the galvanic phenomenon,'' \emph{J.
  Abnorm. Psychol.}, vol.~5, no.~2, pp. 69--74, 1910.

\bibitem{affectivaWhite}
``{Liberate yourself from the lab: Q Sensor measures EDA in the wild},''
  {Affectiva Inc.}, White Paper, 2012.

\bibitem{empatica}
M.~Garbarino \emph{et~al.}, ``Empatica e3 - a wearable wireless multi-sensor
  device for real-time computerized biofeedback and data acquisition,'' in
  \emph{Proc. 4th Int. Conf. Wirel. Mob. Commun. Healthc.}, 2014, pp. 39--42.

\bibitem{benedek10}
{M.~Benedek, and C.~Kaernbach}, ``Decomposition of skin conductance data by
  means of nonnegative deconvolution,'' \emph{Psychophysiology}, vol.~47, pp.
  647--658, 2010.

\bibitem{silveira13}
F.~Silveira \emph{et~al.}, ``Predicting audience responses to movie content
  from electro-dermal activity signals,'' in \emph{Proc. ACM Int. Jt. Conf.
  Pervasive Ubiquitous Comput.}, 2013, pp. 707--716.

\bibitem{lian14}
W.~Lian \emph{et~al.}, ``Modeling correlated arrival events with latent
  semi-{Markov} processes,'' in \emph{Proc. 31st Int. Conf. Mach. Learn.},
  2014, pp. 396--404.

\bibitem{lu12}
H.~Lu \emph{et~al.}, ``{StressSense}: Detecting stress in unconstrained
  acoustic environments using smartphones,'' in \emph{Proc. ACM Conf.
  Ubiquitous Comput.}, 2012, pp. 351--360.

\bibitem{lim97}
C.~Lim \emph{et~al.}, ``Decomposing skin conductance into tonic and phasic
  components,'' \emph{Int. J. Psychophysiol.}, vol.~25, pp. 97--109, 1997.

\bibitem{alexander05}
D.~Alexander \emph{et~al.}, ``Separating individual skin conductance responses
  in a short interstimulus-interval paradigm,'' \emph{J. Neurosci. Methods},
  vol. 146, pp. 116--123, 2005.

\bibitem{bach10}
D.~Bach \emph{et~al.}, ``Dynamic causal modeling of spontenous fluctuations in
  skin conductance,'' \emph{Psychophysiology}, vol.~48, pp. 1--6, 2010.

\bibitem{cvxEDA}
A.~Greco \emph{et~al.}, ``{cvxEDA}: A convex optimization approach to
  electrodermal activity processing,'' \emph{IEEE Trans. Biomed. Eng.},
  vol.~63, no.~4, pp. 797--804, 2016.

\bibitem{chaspari2015sparse}
T.~Chaspari \emph{et~al.}, ``Sparse representation of electrodermal activity
  with knowledge-driven dictionaries,'' \emph{IEEE Trans. Biomed. Eng.},
  vol.~62, no.~3, pp. 960--971, 2015.

\bibitem{mccoy2013achievable}
M.~B. McCoy and J.~A. Tropp, ``The achievable performance of convex demixing,''
  \emph{arXiv preprint arXiv:1309.7478 [cs.IT]}, 2013.

\bibitem{mccoy2014sharp}
------, ``Sharp recovery bounds for convex demixing, with applications,''
  \emph{Found. Comput. Math.}, vol.~14, no.~3, pp. 503--567, 2014.

\bibitem{donoho2006compressed}
D.~L. Donoho, ``Compressed sensing,'' \emph{IEEE Trans. Inf. Theory}, vol.~52,
  no.~4, pp. 1289--1306, 2006.

\bibitem{haupt2010toeplitz}
J.~Haupt \emph{et~al.}, ``Toeplitz compressed sensing matrices with
  applications to sparse channel estimation,'' \emph{IEEE Trans. Inf. Theory},
  vol.~56, no.~11, pp. 5862--5875, 2010.

\bibitem{romberg2009compressive}
J.~Romberg, ``Compressive sensing by random convolution,'' \emph{SIAM J.
  Imaging Sci.}, vol.~2, no.~4, pp. 1098--1128, 2009.

\bibitem{rauhut2012restricted}
H.~Rauhut \emph{et~al.}, ``Restricted isometries for partial random circulant
  matrices,'' \emph{Appl. Comput. Harmon. Anal.}, vol.~32, no.~2, pp. 242--254,
  2012.

\bibitem{yin2010practical}
W.~Yin \emph{et~al.}, ``Practical compressive sensing with {Toeplitz} and
  circulant matrices,'' in \emph{Proc. Vis. Commun. Image Process. Conf.},
  2010, p. 77440K.

\bibitem{berger2010sparse}
C.~R. Berger \emph{et~al.}, ``Sparse channel estimation for multicarrier
  underwater acoustic communication: From subspace methods to compressed
  sensing,'' \emph{IEEE Trans. Signal Process.}, vol.~58, no.~3, pp.
  1708--1721, 2010.

\bibitem{critchley00}
H.~Critchley \emph{et~al.}, ``Neural activity relating to generation and
  representation of galvanic skin conductance responses: A functional magnetic
  resonance imaging study,'' \emph{J. Neurosci.}, vol.~20, no.~8, pp.
  3033--3040, 2000.

\bibitem{healey10}
J.~Healey \emph{et~al.}, ``Out of the lab and into the fray: Towards modeling
  emotion in everyday life,'' in \emph{Proc. 8th Int. Conf. Pervasive Comput.},
  2010, pp. 156--173.

\bibitem{foygel2014corrupted}
R.~Foygel and L.~Mackey, ``Corrupted sensing: Novel guarantees for separating
  structured signals,'' \emph{IEEE Trans. Inf. Theory}, vol.~60, no.~2, pp.
  1223--1247, 2014.

\bibitem{studer2014stable}
C.~Studer and R.~G. Baraniuk, ``Stable restoration and separation of
  approximately sparse signals,'' \emph{Appl. Comput. Harmon. Anal.}, vol.~37,
  no.~1, pp. 12--35, 2014.

\bibitem{cvx}
\BIBentryALTinterwordspacing
M.~Grant and S.~Boyd, ``{CVX}: Matlab software for disciplined convex
  programming, version 2.1,'' Mar. 2014. [Online]. Available:
  \url{http://cvxr.com/cvx}
\BIBentrySTDinterwordspacing

\bibitem{diamond2015matrix}
S.~Diamond and S.~Boyd, ``Matrix-free convex optimization modeling,''
  \emph{arXiv preprint arXiv:1506.00760 [math.OC]}, 2015.

\bibitem{becker2012tfocs}
\BIBentryALTinterwordspacing
S.~Becker \emph{et~al.}, ``{TFOCS}: Templates for first-order conic solvers,''
  2012. [Online]. Available: \url{http://cvxr.com/tfocs/}
\BIBentrySTDinterwordspacing

\bibitem{affectiva}
\BIBentryALTinterwordspacing
``{Affectiva}.'' [Online]. Available: \url{http://www.affectiva.com/}
\BIBentrySTDinterwordspacing

\bibitem{ledalab}
\BIBentryALTinterwordspacing
``{Ledalab MATLAB toolbox}.'' [Online]. Available: \url{http://www.ledalab.de/}
\BIBentrySTDinterwordspacing

\bibitem{MicrosoftBand2}
\BIBentryALTinterwordspacing
``{Microsoft Band SDK}.'' [Online]. Available:
  \url{https://developer.microsoftband.com/bandsdk}
\BIBentrySTDinterwordspacing

\end{thebibliography}


\end{document}